\documentclass[pdflatex,sn-mathphys-num]{sn-jnl}

\usepackage{amssymb}
\usepackage{amsmath}
\usepackage{graphicx}%
\usepackage{multirow}%
\usepackage{amsmath,amssymb,amsfonts}%
\usepackage{amsthm}%
\usepackage{mathrsfs}%
\usepackage[title]{appendix}%
\usepackage{xcolor}%
\usepackage{textcomp}%
\usepackage{manyfoot}%
\usepackage{booktabs}%
\usepackage{algorithm}%
\usepackage{algorithmicx}%
\usepackage{algpseudocode}%
\usepackage{listings}%
\usepackage{thmtools}
\usepackage[noabbrev, capitalise]{cleveref}

\theoremstyle{thmstyleone}%
\newtheorem{theorem}{Theorem}[section]
\newtheorem{proposition}[theorem]{Proposition}%
\newtheorem{lemma}[theorem]{Lemma}

\theoremstyle{remark}%
\newtheorem{remark}{Remark}%
\newtheorem{observation}{Observation}

\theoremstyle{definition}%

\usepackage{hyperref}       
\usepackage{microtype}      

\usepackage{caption}
\usepackage{subcaption}

\newcommand{\cL}{\mathcal{L}}

\newcommand{\cC}{\mathcal{C}}
\newcommand{\R}{\mathbf{R}}

\newcommand{\cD}{\mathcal{D}}

\newcommand{\beq}{\begin{equation}}
\newcommand{\eeq}{\end{equation}}

\newcommand{\rank}{\mathrm{rank}}
\newcommand{\Span}{\mathrm{span}}

\newcommand{\pa}[1]{\left(#1\right)}

\newcommand{\bra}[1]{\left[#1\right]}

\newcommand{\E}{\mathbf{E}}
\DeclareMathOperator{\SoftMax}{Softmax}
\newcommand{\der}[2]{\frac{\partial#1}{\partial#2}}

\usepackage{import}
\usepackage{transparent}

\newcommand{\eliot}[1]{{#1}}
\newcommand{\rita}[1]{{#1}}

\begin{document}

\title[Geometry of Singular Foliations and Learning Manifolds in ReLU Net via the DIM]{Geometry of Singular Foliations and Learning Manifolds in ReLU Networks via the Data Information Matrix}


\author*[1,3]{\fnm{Eliot} \sur{Tron}}\email{eliot.tron@univ-paris1.fr}

\author[2]{\fnm{Rita} \sur{Fioresi}}\email{rita.fioresi@unibo.it}

\affil*[1]{
				\orgname{Ecole Nationale de l'Aviation Civile},
				\orgaddress{\street{7 Avenue Edouard Belin},
								\city{Toulouse},
								\postcode{31055}, 
								\country{France}}}

\affil[2]{\orgdiv{FaBiT}, 
				\orgname{Università di Bologna},
				\orgaddress{\street{via S. Donato 15},
								\city{Bologna},
								\postcode{I-40126}, 
								\country{Italy}}}

\affil[3]{
				\orgdiv{SAMM},
				\orgname{Université Paris 1 Panthéon Sorbonne},
				\orgaddress{\street{31 rue Baudricourt},
								\city{Paris},
								\postcode{75013},
								\country{France}}}

\abstract{
Understanding how real data is distributed in high dimensional spaces is the key to many tasks in machine learning. We want to provide a natural geometric structure on the space of data employing a ReLU neural network trained as a classifier. Through the Data Information Matrix (DIM), a variation of the Fisher information matrix, the model will 
discern a singular foliation structure on the space of data. We show that the singular points of 
such foliation are contained in a measure zero set, and that a local regular foliation exists almost everywhere. 
Experiments show that the data is correlated with leaves
of such foliation.
Moreover we show the potential of our approach for knowledge transfer by analyzing the 
spectrum of the DIM to measure distances between datasets. 
}


\maketitle



\section{Introduction}

The concept of manifold learning lies at the very heart of
the dimensionality reduction question, and
it is based on the assumption that there exists a natural
Riemannian manifold structure on the space of data
\cite{ten1998, hinton2003, ten2000, feff2016}. Indeed, with such
assumption, many geometrical tools become readily available for machine
learning questions, especially in connection with
the problem of knowledge transfer \cite{bozi2020, weiss2016}, as
geodesics, connections, Ricci curvature and similar
\cite{ache2018}.
In particular, the fast developing field of Information
Geometry \cite{amari1998, martens2020}, is now providing with the techniques
to correctly address such questions \cite{maillet}.

However, the practical situation that arises, for instance, in the classification task of benchmark datasets such as MNIST~\cite{lecun1998}, Fashion-MNIST~\cite{fashionmnist} and similar ones is usually too high-dimensional to allow for such straightforward description.
This complexity requires for more sophisticated mathematical modelling, which we shall explore here by replacing the notion of manifold with that of singular foliation. 

Deep Learning
vision classifier models, via their intrinsic hierarchical
structure, offer a natural representation and implicit organization of the input data~\cite{olah}. 
For example, the authors in~\cite{hinton2006} show how a multilayer
encoder network can be successfully employed to transform high
dimensional data into low dimensional code, then to retrieve it via a ``decoder''.


In this paper, we want to provide the data space
of a given dataset with a natural geometrical structure,
and then employ such structure to extract key information.
We will employ a {suitably trained
neural network} model to define a {\sl foliation} structure
in the data space and show {experimentally}
how the training set of our model 
{is strongly correlated with its 
leaves, which are submanifolds of the data space.} 
The mathematical idea of foliation is quite old (see \cite{reeb, ehre}
and refs therein); however, its applications in control theory (see \cite{agra2004} and refs therein)
via sub-Riemannian geometry and machine
learning have only recently become increasingly
important \cite{tron2022, gf2021}.

The foliation structure on the data space, discerned by our model, 
however, is non-standard and presents {\sl singular} 
points. These are points admitting a neighbourhood
where the rank of the distribution, tangent at each point to a leaf of the foliation, 
{changes}. 
Moreover, in the presence of typical
non-smooth activation functions of the network (\textit{e.g.} ReLU, LeakyReLU, Maxpool or similar),
there are also non-smooth points where the manifold structure of the leaf itself is questioned.
However, we prove that both singular and non-smooth points are a measure zero
set in the data space, so that the foliation is almost everywhere regular
and its distribution well defined.
As it turns out {in our experiments, the samples belonging to the 
dataset we train our model with are averagely close}
to the set of singular points. It forces us to
model the data space with a singular foliation that
we call {\sl data foliation} in analogy with the data manifold. 
Applications of singular foliations were introduced 
in connection with control theory  \cite{suss1973};
their study is currently an active
area of investigation \cite{lavau2017}. 
As we shall show in our experiments, singular foliations and distributions provide with an effective tool to single out the samples belonging to the datasets the model was trained with and at the same time discern a notion of distance between different datasets belonging to the same data space.

\subsection*{Main contributions.}
\begin{itemize}
	\item The introduction of a singular geometrical framework to study neural networks with the Data Information Matrix (DIM).
	\item The proof of its well-definednes (a.e.) for ReLU neural networks, given by \cref{thm:main}.
	\item The experimental study of the DIM's eigenvalues and rank on different datasets, which sheds light on its understanding of the network's training.
	\item A toy example (knowledge transfer) that provides a perspective on the relevance of our framework for concrete applications in machine learning.
\end{itemize}

\subsection*{Organization of the paper.}
In \cref{rel-sec}, we recap the
previous literature, closely related to our work.

In \cref{met-sec}, we start by recalling some known facts about Information
Geometry in \cref{info-sec}.
Then, in \cref{fol-sec} we introduce the data information
matrix (DIM), defined via a given model,
and two distributions $\cD$ and $\cD^\perp$ naturally associated
with it, together with some illustrative examples elucidating the
associated foliations and their leaves. 
In \cref{sing-sec} we introduce singular distributions
and foliations and then we prove our main theoretical results
expressed in \cref{lem:ker_P-pp} and \cref{thm:main}.
\cref{lem:ker_P-pp} studies the singularities of the distribution $\cD$. 
\cref{thm:main} establishes that the singular points
for $\cD$ are a measure zero set in the data space. Hence our foliation, 
whose existence relies on mathematical results \cref{thm:frobenius}, that we apply as in \cref{prop:fol},
though singular, acquires geometric significance.
 
Finally in \cref{exp-sec} we elucidate our main results with experiments.
Moreover, we show how 
the foliation structure and its singularities can be exploited 
to determine which dataset the model was trained on, and its distance
from similar datasets in the same data space. 
We also make some ``proof of concept'' considerations
regarding knowledge transfer to show the potential of the mathematical singular
local foliation and distribution structures in data space.


\section{Related Works} \label{rel-sec} 

The use of Neural Networks as tools for both manifold learning
and knowledge transfer has been extensively investigated. {We give a brief overview
on the literature closer to our contribution.}

{\textbf{Latent manifold}}. The question of finding a low dimensional
manifold structure ({\sl latent manifold}) in high dimensional
dataset spaces starts with PCA and similar methods to reach the
more sophisticated techniques as in~\cite{feff2016, maillet}.
The reader can also look at \cite{murphy2022} for a description of the most popular techniques
in dimensionality reduction and
\cite{burges2010} for a complete bibliography or the origins of the subject.
{
Such understanding was applied towards the knowledge transfer questions in 
\cite{cook2007, ben2012} and more recently in~\cite{dikkala2021}. 
{ Moreover, distances between datasets were also studied through Optimal Transport in \cite{Alvarez-Melis2019Geometric, Hua2023DynamicFO}.} 
Another important technique of knowledge transfer is via shared parameters
as in \cite{maurer2016}.
{The idea of using techniques of Information Geometry for 
machine learning started with~\cite{amari1998}. More recently, it was used
for manifold learning
in \cite{maillet} (see also refs therein). 

\textbf{Foliations.} Employing}
foliations to reduce dimensionality is not per se
novel. 
In \cite{gf2021}, the authors introduce some foliation structure in the data space of a neural network, with the aim of approaching dimensionality reduction {exploiting the low rank of the equivalent of the Fisher matrix on the dataspace. 
However, they cannot conclude on the link between the manifold hypothesis and the considered foliation. Dimensionality reduction remains then a conjecture. 
Observations on the low rank of the Fisher matrix,
and its key role in CNN parameter space, appear also in \cite{Sun2017RelativeFI, Sun2024}}.
In \cite{tron2022}, orthogonal foliations in the data
space are used to create adversarial attacks, and provide
the data space with curvature and other Riemannian notions to analyze such attacks.
In \cite{szalai2023} invariant foliations are employed to produce a
reduced order model for autoencoders. Also, in considering singular
points for our foliation, we are led to study the singularities of a neural
network. This was investigated, for the different purpose of network expressibility
in \cite{hanin2019}.{The partitioning of the data space into connected
components after removing non-smooth points of ReLU CNNs has been studied in \cite{Hanin2019ComplexityOL} and 
in \cite{Montfar2014OnTN}.}

\section{Foliations for ReLU Neural Networks} \label{met-sec}

\subsection{The Information Geometry of Neural Networks} \label{info-sec}
\eliot{
In this article, we will focus our attention on neural networks $N_w$ with linear layers, ReLU activation functions and a Softmax at the end, that we call \emph{ReLU neural networks}.
In fact, the results of this article can easily be extended to all piecewise linear activation functions
\rita{such as MaxPool or AveragePool}. 
More precisely, the network is given by $N_w = \SoftMax \circ S_w$ with the score function (or logits) $S_w$ being a composition of linear layers and activation functions defined by:
\begin{equation}\label{eq:chain_rule_score:RNR} 
S_w(x)=L_{W_\ell} \circ \sigma \circ \dots \circ \sigma \circ L_{W_1}
\end{equation}
where $\sigma$ is the ReLU non linearity, $L_{W_i}$ are linear layers (including
bias), $w$ is the collection of all the parameters of the model (\textit{e.g.} the $W_i$'s) and $\ell$ is the total number of linear layers.

The goal of this network will be the classification of the input $x\in\R^d$ between $c$ classes $\{y_1,\ldots,y_c\}$.
The function $N_w$ associates to each input $x$ the probability of belonging to each class, that we will denote by the vector
\begin{equation}\label{probdef}
				p := N_w(x) = \pa{p\pa{y_1\mid x,w}, \ldots, p\pa{y_c \mid x,w}} 
\end{equation}
In these settings, the output of the network gives the parameters of a categorical distribution from which the predicted class $Y$ may be sampled from.
}

We define the \textit{information loss} as:
$I(x,w)= - \ln(p(y|x,w))$ and the \textit{Fisher information matrix} (FIM) \cite{rao1945} as
\eliot{
\begin{equation}\label{fisher}
				F_{i,j}(x,w) := \E_{Y\mid x,w} \bra{\pa{\partial_{w^i} \ln p\pa{Y\mid x,w}}\pa{\partial_{w^j} \ln p\pa{Y\mid x,w}}}, \quad 1\leq i,j\leq N
.\end{equation}
}

In analogy
to (\ref{fisher}) we define the \textit{data information matrix} (DIM):
\eliot{
\begin{equation}\label{dim}
				D_{i,j}(x,w) := \E_{Y\mid x,w} \bra{\pa{\partial_{x^i} \ln p\pa{Y\mid x,w}}\pa{\partial_{x^j} \ln p\pa{Y\mid x,w}}}, \quad 1\leq i,j\leq d 
.\end{equation}
}

\eliot{To compute the DIM in practice, one must then compute the Jacobian matrix of the probability density $p$, \textit{i.e.} the output  of the  network $N_w$, with respect to its input $x$. This is typically achieved through automatic differentiation \cite{baydinAutomaticDifferentiationMachine2017}, implemented in most machine learning libraries.}

As noted in \cite{maillet}, some directions in the parameter or data space, may be more
relevant than others in manifold learning theory.
We have the following result \cite{gf2021}, obtained with a direct calculation.

\begin{proposition}
The Fisher information matrix $F(x, w)$ and 
the data information matrix 
$D(x, w)$ are positive semidefinite symmetric matrices. Moreover, \eliot{if the $\Span$ of a set of vectors denotes the set of all their linear combinations,}
\beq\label{thm:ker}
\begin{array}{c}
  \ker F(x,w)= (\Span_{i=1, \ldots, c}\{
\nabla_w \ln p(y_i|x,w)\})^\perp, \\ \\
\ker D(x,w)= (\Span_{i=1, \ldots, c}\{
\nabla_x \ln p(y_i|x,w)\})^\perp,
\end{array}
\eeq
where the orthogonal $\perp$ is taken with respect to the Euclidean product. 
In particular, \[\rank\ F(x,w) < c \quad \text{and} \quad \rank\ D(x,w) < c.\]
\end{proposition}

Notice that $F(x,w)$ is a $N\times N$ matrix, 
while $D(x,w)$ is a $d\times d$ matrix, where
in typical applications (e.g. classification tasks) $N, d >> c$. 
Hence, both $F(x,w)$ and $D(x,w)$ are, in general, singular matrices, 
with ranks typically low with respect to their sizes, hence
neither $F(x,w)$ nor $D(x,w)$ can provide meaningful metrics
respectively on parameter and data spaces. We shall now focus on
the data space.

Recall that a \emph{distribution} is the association to each point $x$ of a linear subspace of the tangent space at $x$.
The result (\ref{thm:ker}) then
suggests to consider the distribution 
$\cD$:
\beq\label{distr-def}
\begin{array}{c}
\R^d \ni x \mapsto \cD_x:=\Span_{i=1, \ldots, c}\{
\nabla_x \ln p(y_i|x,w)\} \subset T_x\R^d
\end{array}
\eeq
where we assume 
the learning parameters $w \in \R^N$ to be fixed as $x$ varies
{in the dataspace}.
In general a distribution on a manifold $M$ assigns to each point a 
subspace of the tangent space to $M$ at that point (see \cite{tu2008, tu2017} for more
details), in our case $M=\R^d$.

{\begin{observation}\label{rk-d}
The distribution $\cD$ at each point in the dataspace
coincides with the image of $D$ 
at that point as defined in $(\ref{dim})$.
This is a direct check, (see \cite{gf2021}). 
Hence the rank of \rita{the distribution} $\cD$ at $x$, i.e. the dimension of $\cD_x$, and 
the rank \rita{of the DIM} $D$ at x, i.e. the dimension of its image, 
coincide.
\end{observation}}


\subsection{Distributions and Foliations} \label{fol-sec}
Whenever we have a distribution on a manifold $M$, 
it is natural to ask whether it is {\sl integrable}.
A distribution $\cC$ on $M$ is \textit{integrable} if for every
$x \in M$, there exists 
a connected immersed submanifold $N$ of $M$, with $T_y N = \cC_y$ for all $y \in N$.
In other words, the distribution defines at each point the tangent space
to a submanifold. Whenever a distribution $\cC$ is integrable, we have a \textit{foliation},
\rita{denoted by $\cL_\cC$.  The manifold} $M$ becomes a disjoint union of embedded submanifolds
called the \textit{leaves} of the foliation. If $\dim \cC_x=\dim \cC_y$ for
all $x,y \in M$ we say that the foliation, or the corresponding
distribution, has \textit{constant rank}. 
\cref{fol-constrk} illustrates 
a constant rank foliation and its
orthogonal complement foliation, with respect to 
the Euclidean metric in $M=\R^3$ (the ambient space).
\begin{figure}
\begin{center}
\includegraphics[width=0.8\textwidth]{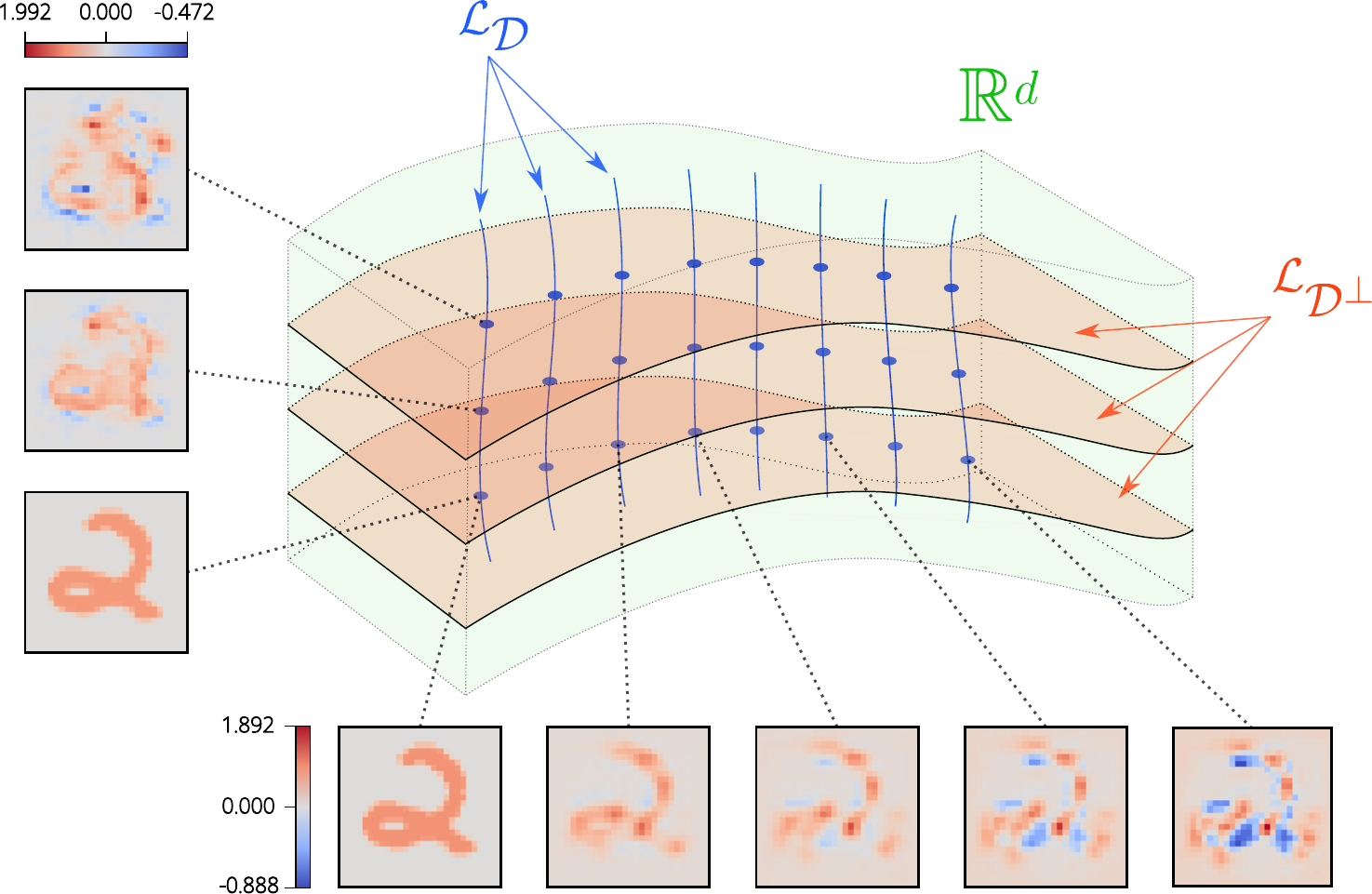}
\end{center}
\caption{Schematic illustration of foliations: $\cL_\cD$ and $\cL_{\cD^\perp}$
denote the set of the leaves in distributions $\cD$ and $\cD^\perp$ respectively. The images represent points in the data space when moving tangentially to $\cL_\cD$ (left) or to $\cL_{\cD^\perp}$ (bottom).}
\label{fol-constrk}
\end{figure}
{In this paper, we focus on the study of the distribution
$\cD$ defined in (\ref{distr-def}), strictly related with the DIM $D$, 
as defined in (\ref{dim}), since $\cD$ is equivalently spanned by
the columns of $D$ (see \cref{rk-d}).}
Under suitable regularity hypothesis, Frobenius theorem \cite{tu2008, tu2017}
provides with a characterization of integrable distributions.

\begin{theorem}[Frobenius Theorem]\label{thm:frobenius}
A smooth constant rank
distribution $\cC$ on a real manifold $M$ 
is integrable if and only if it is involutive,
that is for all vector fields $X,Y \in \cC$, $[X,Y] \in \cC$.
\footnote{If $X$, $Y$ are vector fields 
on a manifold, we define their Lie bracket as $[X,Y](f):=X(Y(f))-Y(X(f))$ for every real function $f$.
}
\end{theorem}

To illustrate this theorem, \cref{fig:xor-data-foliation}
shows two examples of 1-dimensional foliations, obtained
from the distribution $\cD$ defined in \cref{distr-def} of a 
neural network trained on the \texttt{Xor} function, with GeLU\footnote{$\mathrm{GeLU}(x)=x\phi(x)$ with $\phi$ the cumulative distribution function for the Gaussian distribution.} and ReLU\footnote{$\mathrm{ReLU}(x)=\max(0,x)$.} non-linearities.
Notice that, since in both cases $\cD$ is 1-dimensional, it is 
automatically integrable (the bracket of the vector fields is always zero).
\begin{figure}
    \centering
    \begin{subfigure}{0.49\textwidth}
        \centering
        \includegraphics[width=\textwidth]{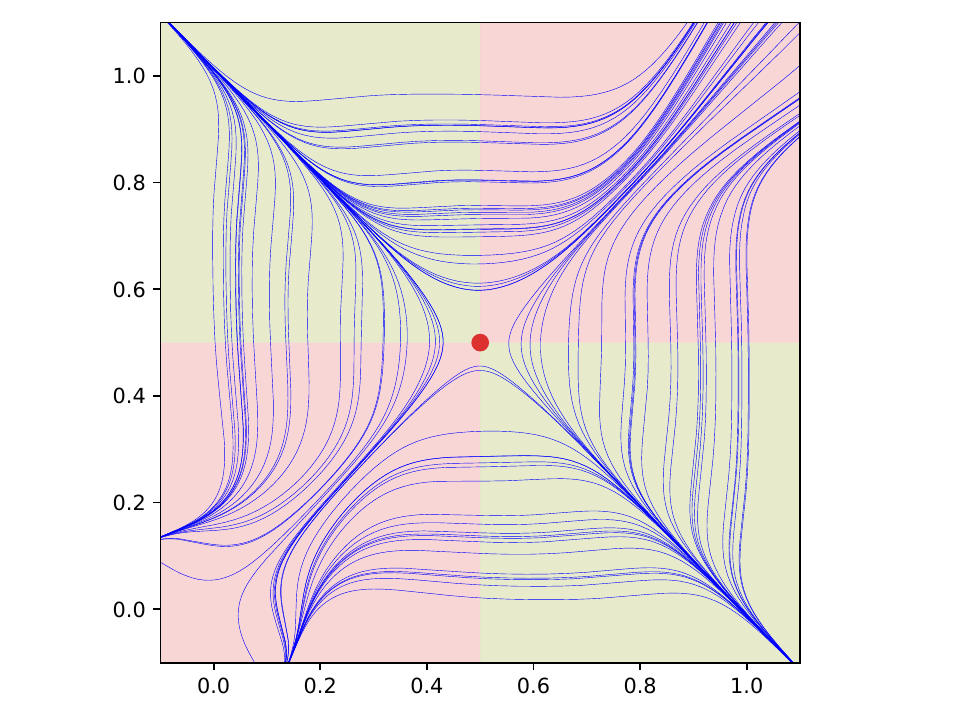}
        \caption{GeLU.}
    \end{subfigure}
    \begin{subfigure}{0.49\textwidth}
        \centering
        \includegraphics[width=\textwidth]{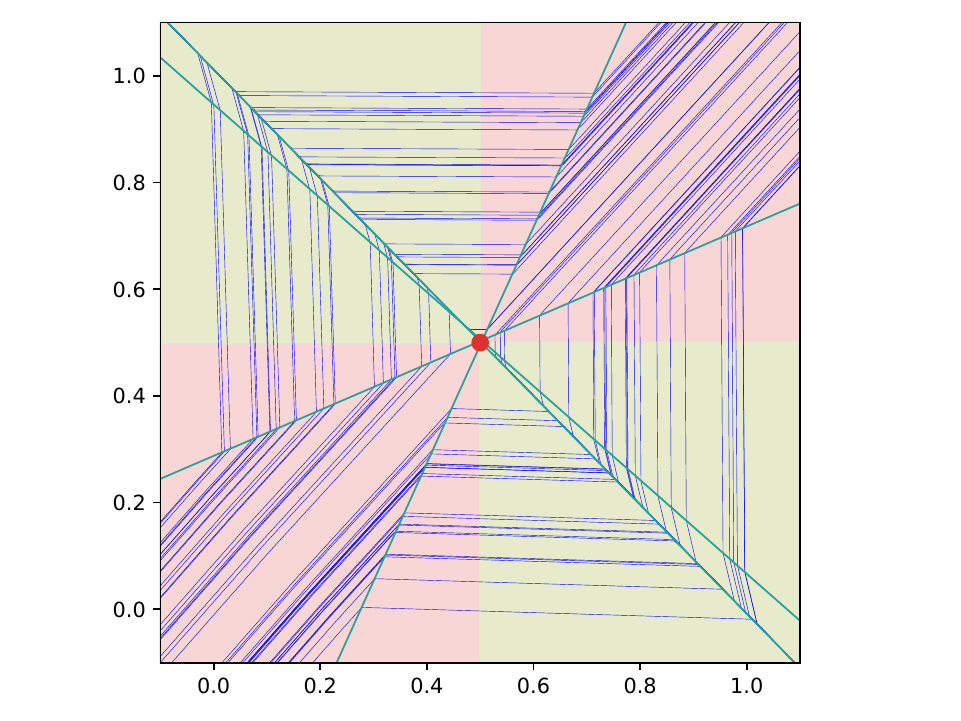}   
        \caption{ReLU.}
    \end{subfigure}
    \caption{{The blue lines are a sample of the {data} foliation defined by the distribution $\cD$ (\ref{distr-def}) for a 
    \texttt{Xor} network. The two classes of the \texttt{Xor} problem are represented in red and green squares underneath. The red dot is a singular point for the foliation. In (b), the green lines are the non-smooth points.}}
    \label{fig:xor-data-foliation}
\end{figure}
{

In a more general setting, especially relevant for the
applications, we have the following result in dimension $d$ with $c\geq 2$.

\begin{proposition}\label{prop:fol}
Let the notation be as above.
For an empirical probability $p$ defined by the output (normalized by softmax) of a ReLU neural network the
distribution $\cD$ in the data space:
$$
x \mapsto \cD_x=\Span_{i=1, \ldots, c}\{
\nabla_x \ln p(y_i|x,w)\}
$$
is locally integrable at smooth points.
\end{proposition}

The proof of this proposition can be found in Appendix~\ref{app:proof_prop_fool}.
We call the foliation resulting from \cref{prop:fol}
the \emph{data foliation}. Its significance is clarified
by the following examples.
}

As an illustration of the distributions $\cD$ and $\cD^\perp$ 
{(orthogonal
computed with respect to the Euclidean metric in $M=\R^d$)}, in
MNIST, 
{we represent in \cref{fol-constrk} pictures of points along paths tangential to $\cD$ (left) or $\cD^\perp$ (bottom).}
{These pictures were}
obtained by projecting the {\sl same} direction on $\cD$ {(left) or} $\cD^\perp$ {(bottom)} at each step.
{Note that moving along a path in $\cL_{\cD^\perp}$ leaves the predicted probability and label invariants, while moving along a path in $\cL_{\cD}$ changes the predicted label to adapt to the new picture.}

\medskip
For a model with GeLU non linearity however, one can see
experimentally that we do not have the involutivity property anymore
for the distribution $\cD$ (see \cref{table1}).
This can be understood as coming from the non-vanishing Hessian matrix of the GeLU function, as detailed in Appendix~\ref{app:proof_prop_fool}. 
Hence, 
there is no foliation whose leaves fill the data
space, naturally associated to {$\cD$} via Frobenius Theorem. To see this
more in detail,
we define the space $\mathcal{V}^D_x$ 
generated by $\cD_x$ 
and the Lie brackets of their generators:
\begin{equation}\label{non-inv}
\begin{array}{c}
\mathcal{V}^D_x=\mathrm{span} \{\nabla_x \ln p(y_i|x,w),
[\nabla_x \ln p(y_j|x,w),
\nabla_x \ln p(y_k|x,w)], i,j,k=1,\dots, c\}\\ \\
\end{array}
\end{equation}
\begin{table}
\caption{Involutivity of the distribution $\cD$\eliot{: dimension averaged on 100 points sampled in $\bra{0,1}^d$.}} 
  \label{table1}
  \centering
  \begin{tabular}{ccc}
    \toprule
Non linearity & dim $\cD_x$  & dim $\mathcal{V}^D_x$ \\ 
\toprule
    ReLU     & 9 & 9  \\ 
    GeLU     & 9 & 44.84   \\ 
    Sigmoid & 9 & 45 \\
    \bottomrule
  \end{tabular}
\end{table}
In \cref{table1} we report averages of the dimensions of the
spaces $\cD_x$, $\mathcal{V}^D_x$ 
for a sample of 100 points $x \in \bra{0,1}^d$ \eliot{with a network trained on MNIST}.
The non involutivity of the distribution
is deduced from the fact that the dimension increases when we take the
space spanned by the distribution and the brackets of its generators.
As we can see, while for the ReLU non linearity
$\cD$ is involutive and we can define the data
foliation, the brackets of
vector fields generating $\cD$ do not lie in $\cD$ for the GeLU and sigmoid non linearities.
Consequently, there is no foliation and 
the sub Riemannian
formalism appears more suitable to describe the geometry in this case. 
We shall not address the question here.

\subsection{Singular Foliations and Main Result}\label{sing-sec}
{In this section, we give our main result 
namely \cref{thm:main}, regarding the singular points of $\cD$.}

{A foliation on a manifold $M$ is a
partition of $M$ into connected immersed submanifolds, called leaves. A foliation is 
\textit{regular} if the leaves have the same dimension, \textit{singular} otherwise.
Notice that the map $x \mapsto \mathrm{dim}(\cL_x )$, which associates to 
$x \in M$ the dimension of its leaf $\cL_x$, is lower semi-continuous, that is,
the dimensions of the leaves in a
neighbourhood of $x$ are greater than or equal to $\mathrm{dim}(\cL_x )$ \cite{lavau2017}. 
Whenever we have an equality, we say that $x$ is \textit{regular}, otherwise
we call $x$ \textit{singular}. It is important to remark that, adhering to the
literature \cite{lavau2017}, the terminology {\sl singular} point
here refers to a point that has a neighbourhood where the leaves have
non constant dimension. 
We can associate a distribution to a foliation by associating to each point, the
tangent space to the leaf at that point. 
Such distribution has constant rank if and only if the foliation is regular. 
Frobenius \cref{thm:frobenius} applies only to the case of constant
rank distributions,
however, there are results extending part of Frobenius theorem to non constant
rank distributions. In \cite{hermann63} and \cite{nagano66}, the authors give sufficient 
conditions for integrability in this more general setting 
(see \cite{lavau2017} for correct attributions on statements and proofs).


In the practical applications under consideration, foliations may also have {\sl non-smooth} points,
that is \eliot{points where the distribution is not continuous, and thus at which the leaf} is not smooth. For example, \cref{fig:xor-data-foliation}
shows a \texttt{Xor} network with ReLU non linearity displaying both
singular and non smooth points.

The singular point of each picture is in the center,
while non smooth points occur only for \texttt{Xor} with ReLU non linearity. As previously stated,
this is a degenerate case, where the involutivity and integrability
of $\cD$ is granted automatically by low dimensionality,
though for (b) in \cref{fig:xor-data-foliation},
we cannot define $\cD$ at non smooth points.

{{In many interesting applications involving neural networks with ReLU non linearity, the distribution $\cD$, generated by the image of the DIM $D$, is both non smooth and singular.} It then becomes important to understand where such
non smooth and singular points are located. The aim of this section is to show
they occupy a measure zero set in the dataspace, so that we can indeed apply
successfully Frobenius theorem in a large portion of the dataspace. In our next
section we shall investigate experimentally their location in dataspace and the importance in the knowledge transfer questions.}


\begin{remark}
Notice that Frobenius \cref{thm:frobenius} and its non constant rank
counterparts, 
apply only in the smoothness hypothesis, while for applications, i.e. the case of ReLU
networks, it is necessary to examine also non smooth foliations.
We plan to explore the non smooth setting more generally in a forthcoming paper.
\end{remark}

We now want to investigate further the singular points of the distribution $\cD$.

Let $p:=N(x)=(p(y_1|x,w), \dots , p(y_c|x,w))$ denote the output of the neural network classifier and 
$J_p(x)=(\nabla_x p(y_i|x,w))$ its Jacobian matrix. So $\nabla_x p(y_i|x,w)$ is the $i$-th row 
(column) of $J_p(x)$.
\eliot{In the following, we might omit the $x$ for shorter computations.}
One can see, by the very definition of $\cD$ (\ref{distr-def}), that $\rank \cD=\rank J_p$.
We assume $w$ to be constant, that is we fix our model. 
\eliot{
Let $P$ be the diagonal matrix defined by:
\[
				P :=\pa{
                \begin{matrix}
								p_1 & & \\
								& \ddots &  \\
								& &  p_C
				\end{matrix}
                }
.\] 
}

\eliot{Recall that} $N(x)=p(y|x,w)=\mathrm{Softmax}\circ S(x)$
{where $S(x)$ represents the score.}
Then we have:
\beq\label{jac1}
J_p(x)=(P-p^tp)J_S(x)
\eeq
where $J_S(x)$ is the Jacobian of the score.

To study the drop of rank for $\cD$ or equivalently for $J_p$, let us first look at the kernel 
of $\left(P-p^tp\right)$.

\begin{lemma}\label{lem:ker_P-pp}
    Let $E_i$ denote the vector with the $i$-th coordinate equal to one, and the others equal to zero.
    \beq
    \ker(P-p^tp)=\Span\{ (1,\dots, 1) \} + \Span\{E_i,~\forall i \text{ such that } p_i=0\}
    \eeq
\end{lemma}

\begin{proof}
Let $u \in \R^c$. Then $u \in \ker(P-p^tp)$ if and only if 
$p_i u_i - p_i \sum_k p_k u_k = 0$. This is equivalent to:
$$
p_i = 0 \quad \text{ or } \quad u_i - \sum_k p_k u_k = 0
\iff  p_i = 0 \quad \text{ or } \quad \sum_k p_k (u_i - u_k) = 0
$$
    The inclusion $\supseteq$ is thus straightforward. To get the other inclusion, let $i_0$ denote the argmax of $u$. Then, $\sum_k p_k(u_{i_0} - u_k)$ is a sum of non negative terms; hence to be equal to zero, must be $p_k(u_{i_0} - u_k) = 0$ for all $k$. Therefore, $u_k = u_{i_0}$ for all $k$ such that $p_k \neq 0$. This is enough to prove the direct inclusion $\subseteq$.
\end{proof}

{ We have the following important observation, that we shall explore more
in detail in our experiments.

\begin{observation}\label{obs-sing}
\cref{lem:ker_P-pp} tells us that the rank of the distribution $\cD$
or equivalently of $J_p(x)=(P-p^tp)J_S(x)$ is lower at points in the data space
where the probability distribution has higher number of $p_i = 0$.
Clearly the points in the dataset, on which our model is trained, are precisely the
points where the empirical probability $p$ is mostly resembling a mass probability distribution.
Hence at such points, we will observe empirically
an average drop of the values of the 
{eigenvalues of the DIM (whose columns generate $\cD$), 
compared to random points in the data space},
as our experiments confirm in \cref{exp-sec}. As we shall see, this 
property characterizes the points in the dataset the model was trained with.
\end{observation}

}

In our hypotheses, since the probability vector $p$ is given by a Softmax function, it cannot have null coordinates. Therefore, \cref{lem:ker_P-pp} states that $\dim \ker (P-p^tp) = 1$ and that the kernel of $P-p^tp$ does not depend on the input $x$.
Thus, the drops in rank of $\cD$ does not depend on $\ker (P-p^tp)$ and can only be caused by $J_S(x)$. 

Now we assume that the score $S$ is a composition of linear layers and activation functions as follows:
\beq\label{s-expr}
S(x)=L_{W_\ell} \circ \sigma \circ \dots \circ \sigma \circ L_{W_1}
\eeq
where $\sigma$ is the ReLU non linearity, $L_{W_i}$ are linear layers (including
bias) and $\ell$ is the total number of linear layers.
We denote the output of the $k$-th layer:
\[
f_k(x) = L_{W_k}\circ \sigma \circ \dots \circ \sigma \circ L_{W_1}(x)
\]
Let us define, for a subset $U$ in $\R^d$:
\begin{equation}
    Z_U = \left\{ x\in U \text{ such that } \exists i,~ x_i=0\right\}
\end{equation}

\begin{lemma} \label{lem:rk-js}
The set of singular points of $J_S$, the Jacobian of $S$,  is a subset of:
    \begin{equation}
        \bigcup_{k=1}^{\ell-1} f^{-1}_k (Z_{f_k(M)})
        = \bigcup_{k=1}^{\ell-1} \bigcup_{i=1}^{\dim f_k(M)} \left\{x\mid f_k(x)_i = 0\right\}
    \end{equation}
    This set is the finite union of closed null spaces, 
 thus of zero Lebesgue measure.
\end{lemma}

\begin{proof} 
A short calculation based on the expression of $S$ (\ref{s-expr})
gives:
\beq
\begin{array}{lcl}
    J_S(x) & = & W_\ell J_\sigma\left(f_{\ell - 1 }(x)\right)W_{\ell - 1} J_\sigma \left(f_{\ell - 2}(x) \right) \dots J_\sigma(f_1(x)) W_1 
\end{array}\label{eq:jacS}
\eeq
Notice that:
\begin{equation}
    \left(J_{\text{ReLU}} (x)\right)_{i,j} = 
    \begin{cases}
        \delta_{i,j} & \text{if } x_i > 0 \\
        0 & \text{if } x_i < 0
    \end{cases}
\end{equation}

Hence, the set $Z_U$ represents the singular points of $J_{\text{ReLU}}$ on the domain $U$.

\eliot{First, note that $f_k$ are piecewise affine functions, \textit{i.e.} divided into a finite number of region on which $f_k$ is affine. Let us focus on one of these regions. If $(f_k)_i$ is non-constant, then the equation $f_k(x)_i = 0$ defines a hyperplane of dimension $< d = \dim M$ as the kernel of a non-constant affine function. If $(f_k)_i$ is constant, it is either equal to a non-zero value and then the set $\{x\mid f_k(x)_i = 0\}$ is empty; or $(f_k)_i$ is constant equal to zero.
This last case may be problematic as it implies that a whole region (with positive measure) could consist of degenerate and non-smooth points. However, note that this only occurs for a specific value of the bias $(b_k)_i$. This is occurring with probability zero after training\footnote{To elaborate on this argument, even if this situation were to occur after training, one could escape from it by adding $\varepsilon>0$ to $(b_k)_i$, thereby breaking the problematic equality. With a sufficiently small $\varepsilon$, the network's predictions should remain close enough to the original ones by continuity.}, therefore we can safely discard this case.} 
\eliot{In conclusion, the set of singular points of $J_S$ is contained in a union of hyperplanes with dimension $< d=\dim M$.} 
\end{proof}



Now we see that 
singular points 
occur on a (Lebesgue) measure zero set \rita{given by the union of the
hypersurfaces of equation $(f_k)_i=0$, which is a piecewise linear condition (see proof
of \cref{lem:rk-js}).} 

\begin{theorem}\label{thm:main}
Let the notation be as above.
Consider the distribution $\cD$:
\beq
\R^d \ni x\mapsto \cD_x=\mathrm{span}\{\nabla_x p(y_i|x,w), \, i=1, \dots c\}
\eeq
where $p$ is an empirical probability given by Softmax and a score
function $S$ consisting of a sequence of linear layers and ReLU activations.
Then, 
its singular points
(i.e. points where $\cD$ changes its rank) 
are a closed null subset of $\R^d$ contained in a union
hypersurfaces.
\end{theorem}

\begin{proof} 
The 
singular points of $\cD$ coincide, by \cref{lem:ker_P-pp} with
the points where $J_S$ the Jacobian of $S$ as in (\ref{s-expr}) changes its
rank. By \cref{lem:rk-js} this occurs in a union of
hypersurfaces, \rita{given by the equations $(f_k)_i=0$ (see proof of \cref{lem:rk-js})}.
\end{proof}


As a consequence of the proofs of \cref{lem:rk-js}
and \cref{thm:main}, the singular points of the distribution $\cD$ are contained 
in the non smooth points and such points are contained in a measure zero subset of the data space.
Hence, if we restrict ourselves to the open set complementing such measure zero set, 
we can apply Frobenius Theorem {\ref{thm:frobenius}
to the distribution $\cD$ to get a foliation. Since
it is of significance in the dataspace, as we are going to elucidate in the next section, we
call such foliation the \textit{data foliation}.}

{ \begin{remark} \cref{thm:main} holds also in the more general hypothesis
of piecewise linear activation functions, in place of ReLU.\end{remark}}

{\begin{observation}\label{obs-sing2}
There is empirical evidence \cite{Hanin2019ComplexityOL} that, removing the measure zero
set of non smooth points as in \cref{thm:main}, partitions the dataspace into disjoint connected components. In each component the distribution $\cD$ has constant
rank, hence the foliation restricted to one component is regular. However the global
foliation on the dataspace, granted locally by Frobenius Theorem
\ref{thm:frobenius}, may vary its rank as we move from one connected component to
another, as the experiments in our next section suggest. Hence the notion of singular
foliation, i.e. a foliation with leaves of different dimensions, become essential to properly model the
dataspace, in which we train the neural network.
\end{observation}}

\subsection*{Interpretation of Singularities.}
The dimension of a leaf from the data foliation indicates the number of degrees of freedom available to modify the output prediction by adjusting the input.
When the rank of $\cD_x$ (or equivalently $D(x)$) is lower than $c-1$, then the changes to the prediction are constrained.
For instance, if the rank of $\cD_x$ is only $1$, then any local change to the input will produce a displacement along a one dimensional path in the output space.
A singular point is then where leaves of different dimensions meet in the input space, adding (or removing in the opposite direction) degrees of freedom to the output's modifications.

\section{Experiments}\label{exp-sec}

{In this section we look at some experiments to elucidate the role
of the singular and non-smooth points (see \cref{thm:main}) for the distribution
$\cD$, generated by the image of the DIM $D$ (see Def. (\ref{dim})),
in knowledge transfer questions. The emergence of a singular 
foliation structure in an dense open set of the dataspace, granted by \cref{thm:main},
via Frobenius \cref{thm:frobenius} (see also \cref{obs-sing2})
is related with the inner structure of datasets in such dataspace.
Indeed, we will show that an effective way to measure the "distance" between
different datasets embedded in the same dataspace, 
is related with the rank of $\cD$, or equivalently $D$,
and the (relative) magnitude of $D$'s eigenvalues.}

\medskip
We perform our experiments on the following datasets: MNIST~\cite{lecun1998}, 
Fashion-MNIST~\cite{fashionmnist}, KMNIST~\cite{kmnist} and EMNIST~\cite{emnist}, letters only, that we denote with Letters. 
We also create a dataset that we call
CIFARMNIST: it is the CIFAR10 dataset~\cite{cifar} cropped and transformed to 
be $28\times 28$ gray-scale pictures.
Finally, we create a randomly generated dataset called \emph{Noise} as a baseline, with input points uniformly sampled in the $[0,1]^{28\times 28}$ cube.

Our neural network is similar to LeNet, with 
two convolutional layers, followed by an Averagepool and two linear layers 
with ReLU activation functions
, {see Appendix~\ref{app:NN_archi} for more details on the architecture}.
This is slightly
more general than 
our hypotheses in \cref{sing-sec}.
The model is then trained on MNIST, reaching $98\%$ of accuracy.
\footnote{All the code used in this section is available at \url{https://github.com/eliot-tron/SingFolDIM}.}

In \cref{fig:evalues1} we compute the {DIM} and we measure its rank
by looking at its { eigenvalues} 
for 100 sample points in the data space $\R^{784}$ of the MNIST dataset and on 100 { uniformly random points in $[0,1]^{784}$.}
The statistical significance of these experiments is detailed in \cref{fig:evalues}.
\begin{figure}
\centering
    \begin{subfigure}{0.49\textwidth}
        \centering
        \includegraphics[width=\textwidth]{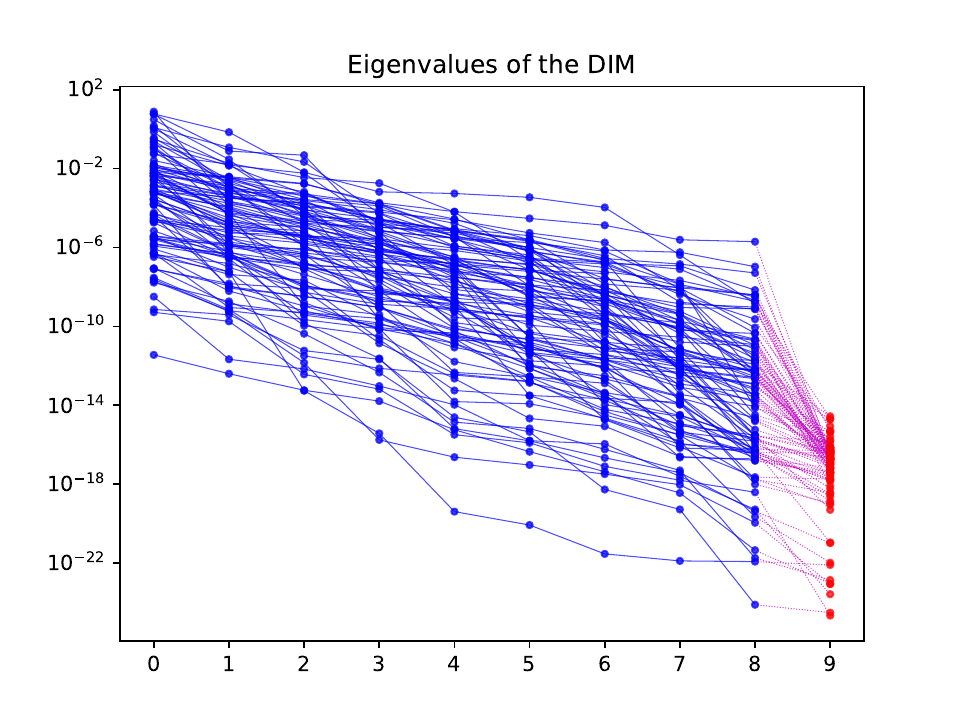}
        \caption{MNIST data points.}
    \end{subfigure}
    \begin{subfigure}{0.49\textwidth}
        \centering
        \includegraphics[width=\textwidth]{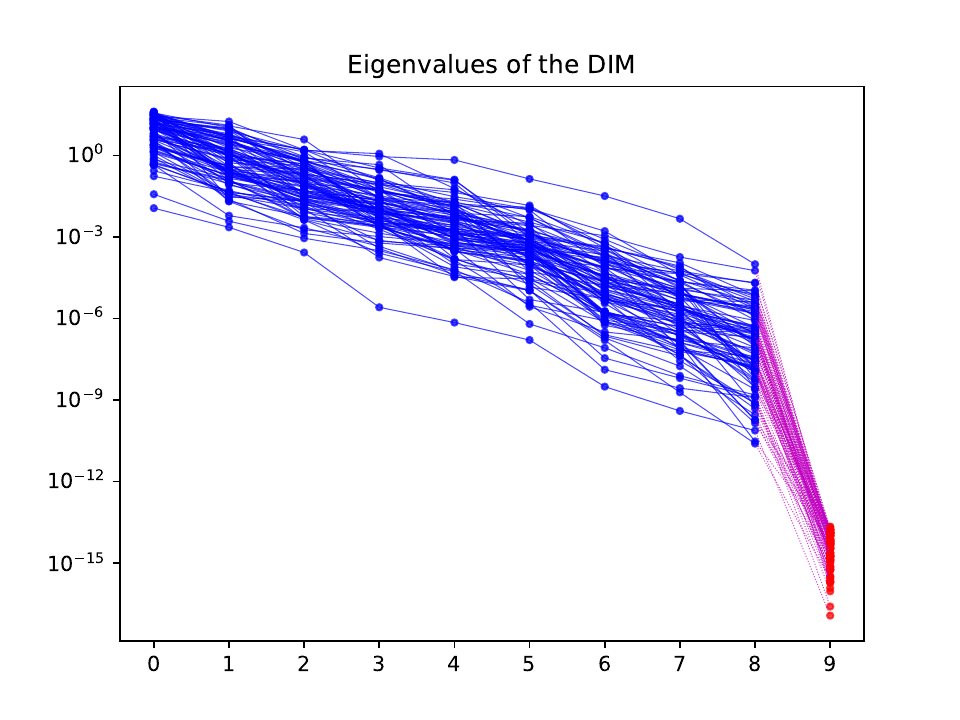}   
        \caption{Uniformly random points (Noise).} 
    \end{subfigure}
    \caption{DIM eigenvalues sorted by descending order evaluated on 100 points. Each blue line correspond to one data point.}
    \label{fig:evalues1}
\end{figure}
We see clearly that on points in the training set
the { eigenvalues} 
of {the DIM $D$}, 
are {smaller}.  
{Since {$\rank D=\rank \cD$} at each point (see \cref{fol-sec}), the drop
of the {eigenvalues of $J_p$} reflects the change
in rank of the distribution $\cD$ in the proximity of datapoints, i.e. the points on which
network was trained. This provides evidence for modeling the dataspace via
a singular foliation (see \cref{obs-sing2}.
and \cref{lem:ker_P-pp}). }

\begin{figure}
        \centering
        \includegraphics[width=0.8\textwidth]{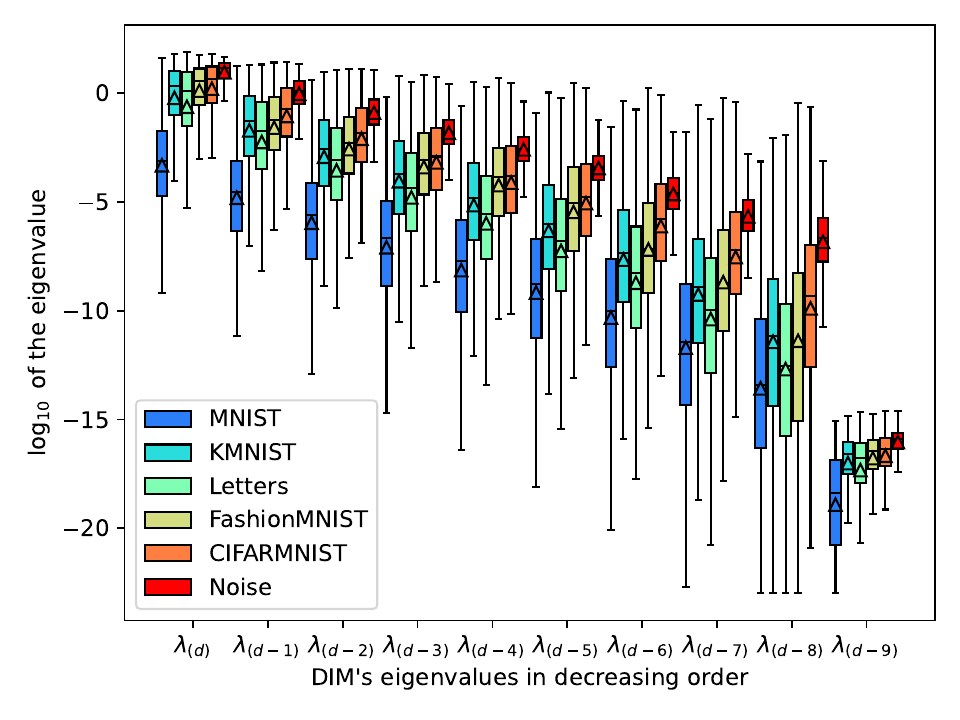}
    \caption{DIM eigenvalues sorted by descending order, evaluated on 10,000 points for each dataset, for a network trained on MNIST.}%
    \label{fig:evalues}
\end{figure}

We compute the eigenvalues $\pa{\lambda_i}_{1\leq i \leq d}$ of the DIM and sort them as such:
\[
				\lambda_{(1)} \leq \ldots \leq \lambda_{(d)}
.\] 
Then, we report the eigenvalues $\lambda_{(d)},\ldots,\lambda_{(d-9)}$, sorted in descending order\footnote{Eigenvalues $\lambda_{(d-10)},\ldots,\lambda_{(1)}$ are not plotted because $\rank D < c=10$, and therefore are equal to zero. $\lambda_{(d-9)}$ is plotted only for reference of what the numerical zero is.} (logarithmic scale), for different datasets in Fig.
\ref{fig:evalues}.

The vertical segment for each eigenvalue and each
dataset represents the values for 80\% of the samples, while the colored area
represents the values falling in between first and third quartile. The horizontal
line represents the median and the triangle represents the mean.
The points in MNIST, the training
dataset, are clearly identifiable by looking at the colored area.
On the opposite, the randomly generated points in Noise yield higher eigenvalues with fewer variations.
This difference in eigenvalues between the various datasets prompts us to consider using the DIM to measure how different the datasets are from the network's perspective.

To provide a point of comparison to our analysis, we have included with \cref{fig:evalues_random_weights} a similar plot to \cref{fig:evalues} but for a neural network with random weights (\textit{i.e.} before any training).
As can be seen, the singularities appear during training when they become concentrated around points similar to those in the training set.
\begin{figure}
        \centering
        \includegraphics[width=0.8\textwidth]{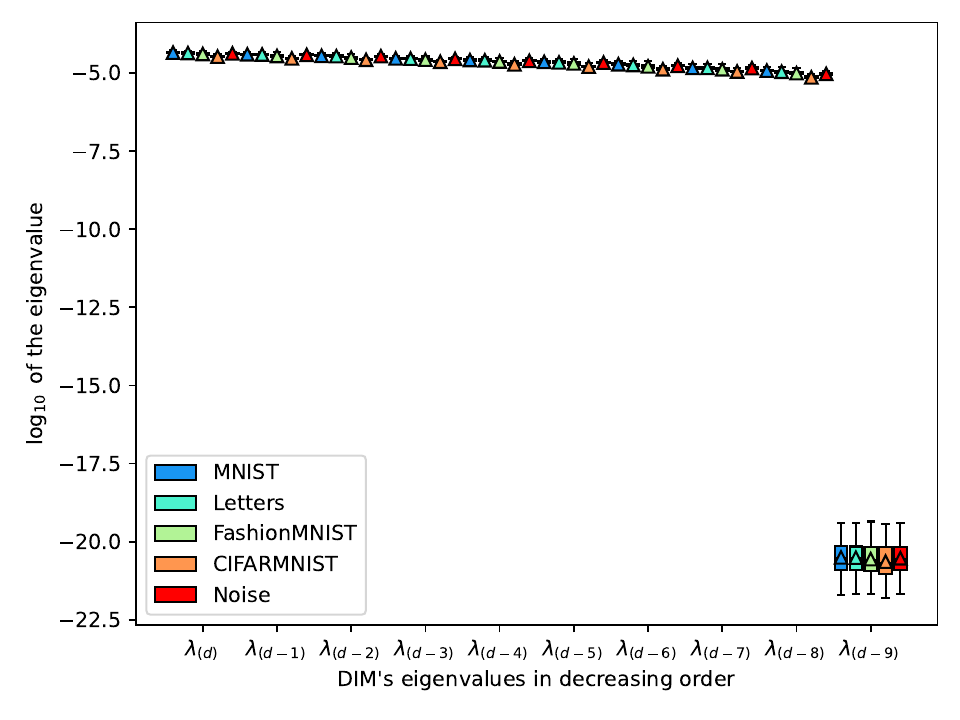}
    \caption{DIM eigenvalues sorted by descending order, evaluated on 10,000 points for each dataset, for a network with random weights.}%
    \label{fig:evalues_random_weights}
\end{figure}


\begin{figure}
    \centering
    \includegraphics[width=0.7\textwidth, height=0.3\textheight]{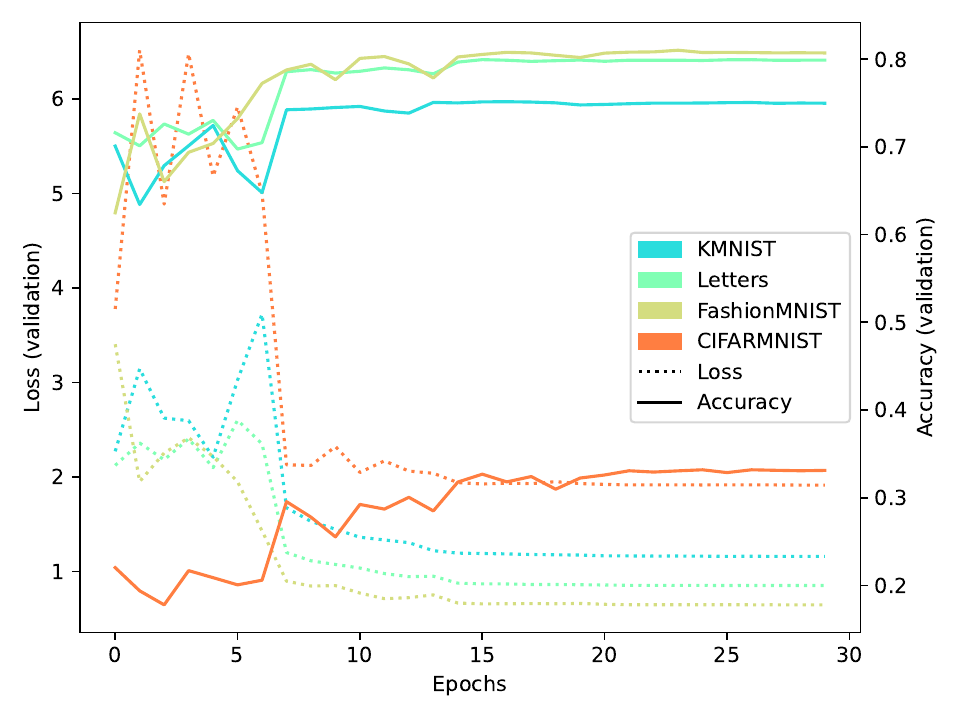}
    \caption{Loss and accuracy after knowledge transfer starting from the weights of a ReLU network trained on MNIST ($98\%$ of accuracy) and retraining only the last linear layer. }
    \label{fig:loss-and-acc}
\end{figure}

We perform a proof of concept knowledge transfer experiment by retraining the
last linear layer of our model on different datasets. {This experiment was selected as it relies on the degree of similarity between datasets from the perspective of the model.}
{
We report in \cref{fig:loss-and-acc} the evolution of the loss (dotted) and accuracy (plain) of the network after 30 epochs of retraining on 4 different datasets. 
We notice a correspondence between the magnitude of the eigenvalues and the validation accuracy suggesting to explore in future works the relationship between DIM, foliations and knowledge transfer. Indeed, CIFARMNIST performs worst than the others and has higher eigenvalues in average.
}

{With the magnitude of the eigenvalues of the DIM, we find similar results as the authors of~\cite{Alvarez-Melis2019Geometric} regarding the similarity between MNIST and the other datasets.}

\rita{
We report in \cref{table2} the median of highest {($\lambda_{(d)}$)} and lowest {non-zero ($\lambda_{(d-8)}$)} DIM eigenvalues in 
logarithmic scale, {$\Delta$ the average length of the vertical segments in \cref{fig:evalues}}, and the validation accuracy after
retraining of our CNN. We see a correspondence between
the median of the lowest {non-zero} eigenvalue and the validation accuracy, suggesting to explore
in future works the relation between DIM, foliations and knowledge transfer.
{The significance of the last non-zero eigenvalue of the DIM lies in its correspondence with the rank of the foliation. Indeed, the eigenvalue $\lambda_{(d-8)}$ being (close to) zero means that the rank drops at (close to) the data point. This supports our conjecture that similar datasets (i.e. with high knowledge transfer potential) will lead to data leaves with similar geometric properties. Moreover it suggests a definition of dataset distance based on such
eigenvalue magnitude.}}

 \begin{table}
 \caption{\rita{Parameters for Knowledge Transfer (log scale)}}
   \label{table2}
   \centering   
   \begin{tabular}{lccccc}
     \toprule
 Dataset & {$\lambda_{(d)}$}  & {$\lambda_{(d-8)}$} &  $\Delta$ & DIM Trace & Val. Acc. \\ 
 \toprule
     MNIST         & -3.10 & -13.40  & 10.07 & -1.52 & 98\% \\
     Fashion-MNIST & 0.57  & -11.54 & 11.56 &  0.12 & 81\% \\
     Letters       & 0.09  & -12.56 & 11.99 &  0.48 & 80\% \\
     KMNIST        & 0.34  & -11.17 & 10.98 &  0.37 & 75\% \\
     CIFARMNIST    & 0.67  & -9.34 & 9.44 &  0.27 & 33\% \\
     Noise         & 1.17  & -6.66 & 7.58 &  0.27 & NA   \\
     \bottomrule
   \end{tabular}
 \end{table}


\section{Conclusions}
{We propose to complement the notion of data manifolds
with the more general one of {\sl data foliations}. The 
(integrable) distribution $\cD$, defined 
via the data information matrix (DIM), a generalization of the Fisher
information matrix to the space of data, indeed allows for the partition of the data space according to the leaves of
such a foliation via the Frobenius theorem.
Examples and experiments show a correlation of data points
with the leaves of the foliation:
moving according to the distribution $\cD$, i.e.
along a leaf, the model gives a meaningful label, 
while moving in the orthogonal directions leads to greater and greater
classification errors.
The data foliation is however both singular (drop in rank)
and non smooth. We prove that singular points are contained into a set
of measure zero,
hence making the data foliation significant in the data space.
We show that points in the dataset the model was trained with have
lower DIM eigenvalues, so that the distribution $\cD$
allows successfully to determine whether a sample of points belongs or
not to the dataset used for training. We make such explicit 
comparison with similar datasets (i.e. MNIST versus FashionMNIST, KMNIST etc).
{Then, we use the average magnitude of DIM's eigenvalues to understand the distance between data sets.}
We test our proposed distance by retraining our model on
datasets belonging to the same data spaces and checking the
validation accuracy. Our results are not quantitatively conclusive in this
regard, {but they represent a "proof of concept", to encourage further investigation. We believe that our empirical examples} 
show a great promise as a first step to
go beyond the manifold hypothesis and exploiting the theory
of singular foliations to perform dimensionality reduction and
knowledge transfer.}

\section*{Acknowledgments}
This research was supported by Gnsaga-Indam, by COST Action CaLISTA CA21109, HORIZON2020 CaLIGOLA MSCA-2021-SE-01-101086123 MSCA-DN CaLiForNIA-101119552, PNRR
MNESYS, PNRR National Center for HPC, Big Data and Quantum Computing, PNRR SymQUSEC,
INFN Sezione Bologna
and
PEPR ``ANR-23-PEIA-0004''.

\begin{appendices}
\section{{Neural Network Architecture}}
\label{app:NN_archi}
{
In this section, you can find on \cref{tab:MNIST_architecture} and \cref{fig:lenet} the architecture of the ReLU network trained on MNIST used for the experiment in \cref{exp-sec}. }

\begin{table}[ht]
    \centering
    \caption{Architecture of the neural network trained on MNIST.}
    \label{tab:MNIST_architecture}
    \begin{tabular}{rl}
    \toprule
    No.  & Layers (sequential) \\
    \midrule
		(0): & \texttt{Conv2d(1, 32, kernel\_size=(3, 3), stride=(1, 1))} \\
  (1): & \texttt{ReLU()}\\
  (2): & \texttt{Conv2d(32, 64, kernel\_size=(3, 3), stride=(1, 1))}\\
  (3): & \texttt{ReLU()}\\
  (4): & \texttt{AvgPool2d(kernel\_size=2, stride=2, padding=0)}\\
  (5): & \texttt{Flatten()}\\
  (6): & \texttt{Linear(in\_features=9216, out\_features=128, bias=True)}\\
  (7): & \texttt{ReLU()}\\
  (8): & \texttt{Linear(in\_features=128, out\_features=10, bias=True)}\\
  (9): & \texttt{Softmax()}\\
  \bottomrule
    \end{tabular}
\end{table}

\begin{figure}[ht]
    \centering
    \includegraphics[width=0.7\textwidth]{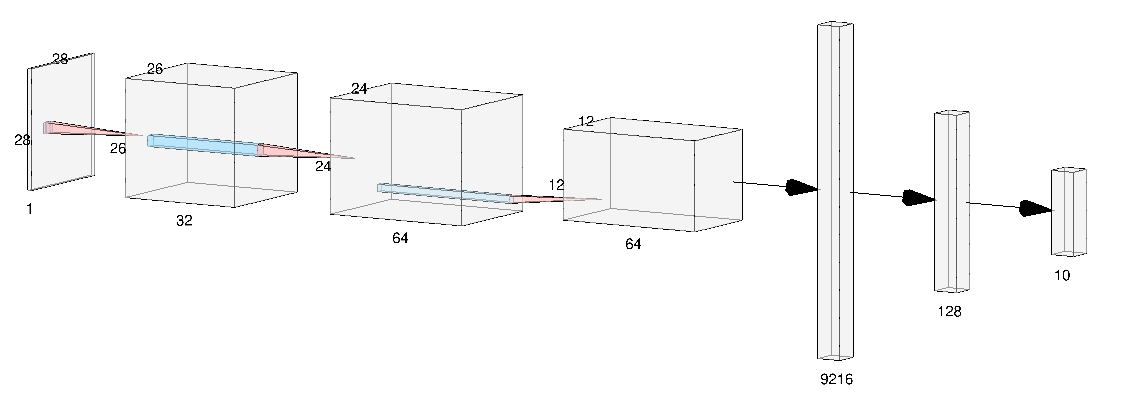}
    \caption{The structure of our CNN - picture created with~\cite{lenail}.}
    \label{fig:lenet}
\end{figure}

\section[Proofs]{Proof of \cref{prop:fol}}\label{app:proof_prop_fool}

Note that these computations can also be found in \citep{gf2021,tronCartanMovingFrames2024} under different forms.

First, recall that
\[
				\bra{u^k\partial_k, v^l \partial_l} = \pa{u^k \der{v^l}{x^k} - v^k \der{u^l}{x^k} } \partial_l
.\]
Thus, if $1\leq a,b \leq c-1$, $u^k = \der{p_a}{x^k}$ and $v^l = \der{p_b}{x^l}$ we get:
\begin{align}
\bra{u, v} &= \pa{\sum_k \der{p_a}{x^k} \der{^2 p_b}{x^k\partial x^l} -
\sum_k\der{p_b}{x^k} \der{^2 p_a}{x^k\partial x^l} } \partial_l \nonumber \\
					 &=  H\pa{p_b} u -  H\pa{p_a} v \label{eq:lie_bracket_D}
\end{align}
with $H(p_a)$ the Hessian matrix of $p_a = N(x)_a$.

Let $s$ represent the score vector, \textit{i.e.} the output of the neural network right before going through the softmax function.
Then, for all $i,j=1,\ldots,c$,
\begin{equation}\label{eq:nabla_p_a:CC}
\partial_{j} p_i = \sum_k p_i \pa{\delta_{ik} - p_k} \partial_{j} s_k 
.\end{equation} 
This is straightforward from the definition of the softmax function.

We can then express the Hessian matrix of $p_a$ with respect to $s$ and its derivative as follow.
\begin{align*}
 H\pa{p_a}_{ij} &= \partial_{i}\pa{\partial_{j} p_a } = \partial_{i}\pa{\sum_k p_a \pa{\delta_{ik} - p_k} \partial_{j} s_k} \\
								&= \partial_{i} p_a \sum_k \pa{\delta_{ak} - p_k} \partial_{j} s_k + p_a \sum_k \partial_{i}\pa{\pa{\delta_{ak}-p_k} \partial_{j} s_k} \\
								&= \partial_{i} p_a \sum_k \pa{\delta_{ak} - p_k} \partial_{j} s_k + p_a \sum_k \pa{\delta_{ak} - p_k}  H\pa{s_k} -
p_a \sum_k \partial_{i} p_k \partial_{j} s_k\\
								&= \sum_k  \pa{\pa{\delta_{ak} - p_k}\partial_{i} p_a  - p_a \partial_{i} p_k} \partial_{j} s_k + p_a \sum_k \pa{\delta_{ak} - p_k}  H\pa{s_k}
.\end{align*}

In the case of ReLU activation functions, the second order derivative of the score $s$ vanishes almost everywhere\footnote{This question is the focus of \cref{thm:main}.}.
We can thus safely discard the last term.

Now, we compute $H(p_a) \nabla_x p_b$ by plugging in the above expression, taking advantage of $H(p_a)$'s symmetry:
\begin{align*}
				\pa{H\pa{p_a} \nabla_x p_b}_j &= \sum_i \pa{ \sum_k \pa{\pa{\delta_{ak} - p_k}\partial_{j} p_a  - p_a \partial_{j} p_k } \partial_{i} s_k } \partial_{i} p_b \\
				\begin{split}
												&= \sum_i \pa{ \sum_k \pa{\delta_{ak} - p_k}\partial_{j} p_a \partial_{i} s_k} \partial_{i} p_b \\
												&	\qquad - p_a \sum_{i}\pa{\sum_k \partial_{j} p_k \partial_{i} s_k} \partial_{i} p_b \\
				\end{split}\\
					&=  \pa{\sum_{k,i} \pa{\delta_{ak} - p_k}\partial_{i} s_k \partial_{i} p_b} \partial_{j} p_a  - p_a \sum_{k,i}  \pa{\partial_{i} s_k \partial_{i} p_b} \partial_{j} p_k 
.\end{align*}

Therefore,
\[
				H(p_a)\nabla_x p_b = \pa{\sum_{k,i} \pa{\delta_{ak} - p_k}\partial_{i} s_k \partial_{i} p_b} \nabla_x p_a  - p_a \sum_{k,i}  \pa{\partial_{i} s_k \partial_{i} p_b} \nabla p_k 
.\] 
This conclude the proof by showing that, for ReLU activation functions, we have $H(p_a)\nabla_x p_b \in \Span_{i=1, \ldots, c}\{\nabla_x \ln p(y_i|x,w)\}$
Thus the distribution $\cD$ is locally involutive and Frobenius theorem applies.

\end{appendices}

\bibliography{main}

@article{baydinAutomaticDifferentiationMachine2017,
  title = {Automatic Differentiation in Machine Learning: A Survey},
  shorttitle = {Automatic Differentiation in Machine Learning},
  author = {Baydin, At{\i}l{\i}m G{\"u}nes and Pearlmutter, Barak A. and Radul, Alexey Andreyevich and Siskind, Jeffrey Mark},
  year = {2017},
  month = jan,
  journal = {J. Mach. Learn. Res.},
  volume = {18},
  number = {1},
  pages = {5595--5637},
  issn = {1532-4435}
}

@inproceedings{Alvarez-Melis2019Geometric,
 author = {Alvarez-Melis, David and Fusi, Nicolo},
 booktitle = {Advances in Neural Information Processing Systems},
 editor = {H. Larochelle and M. Ranzato and R. Hadsell and M.F. Balcan and H. Lin},
 pages = {21428--21439},
 publisher = {Curran Associates, Inc.},
 title = {Geometric Dataset Distances via Optimal Transport},
 volume = {33},
 year = {2020},
 address = {Red Hook, NY, United States}
}

@article{ache2018,
title = {Ricci curvature and the manifold learning problem},
journal = {Advances in Mathematics},
volume = {342},
pages = {14-66},
year = {2019},
issn = {0001-8708},
doi = {https://doi.org/10.1016/j.aim.2018.11.001},

author = {Antonio G. Ache and Micah W. Warren}
}

@article{agra2004,
author = {Agrachev, Andrei and Sachkov, Yuri},
year = {2004},
title = {Control Theory from the Geometric Viewpoint},
publisher = {Springer Berlin, Heidelberg},
isbn = {978-3-642-05907-0},
doi = {10.1007/978-3-662-06404-7}
}

@ARTICLE{amari1998,

  author={Amari, Shun-ichi},

  journal={Neural Computation}, 

  title={Natural Gradient Works Efficiently in Learning}, 

  year={1998},

  volume={10},

  number={2},

  pages={251-276},

  doi={10.1162/089976698300017746}}

@InProceedings{ben2012,
  title = 	 {Deep Learning of Representations for Unsupervised and Transfer Learning},
  author = 	 {Bengio, Yoshua},
  booktitle = 	 {Proceedings of ICML Workshop on Unsupervised and Transfer Learning},
  pages = 	 {17--36},
  year = 	 {2012},
  editor = 	 {Guyon, Isabelle and Dror, Gideon and Lemaire, Vincent and Taylor, Graham and Silver, Daniel},
  volume = 	 {27},
  series = 	 {Proceedings of Machine Learning Research},
  address = 	 {Bellevue, Washington, USA},
  month = 	 {02 Jul},
  publisher =    {PMLR}
}

@article{bozi2020,
author = {Bozinovski, Stevo},
year = {2020},
month = {09},
pages = {},
title = {Reminder of the First Paper on Transfer Learning in Neural Networks, 1976},
volume = {44},
journal = {Informatica},
doi = {10.31449/inf.v44i3.2828}
}

@article{burges2010,
author = {Burges, Christopher},
year = {2010},
month = {01},
pages = {},
title = {Dimension Reduction: A Guided Tour},
volume = {2},
journal = {Foundations and Trends in Machine Learning},
doi = {10.1561/2200000002}
}

@InProceedings{cook2007,
  title = 	 {Visualizing Similarity Data with a Mixture of Maps},
  author = 	 {Cook, James and Sutskever, Ilya and Mnih, Andriy and Hinton, Geoffrey},
  booktitle = 	 {Proceedings of the Eleventh International Conference on Artificial Intelligence and Statistics},
  pages = 	 {67--74},
  year = 	 {2007},
  editor = 	 {Meila, Marina and Shen, Xiaotong},
  volume = 	 {2},
  series = 	 {Proceedings of Machine Learning Research},
  address = 	 {San Juan, Puerto Rico},
  month = 	 {21--24 Mar},
  publisher =    {PMLR}
}

@misc{dikkala2021,
      title={For Manifold Learning, Deep Neural Networks can be Locality Sensitive Hash Functions}, 
      author={Nishanth Dikkala and Gal Kaplun and Rina Panigrahy},
      year={2021},
      eprint={2103.06875},
      archivePrefix={arXiv},
      primaryClass={cs.LG},
       
}

@inbook{ehre,
 ISBN = {9781487573300},
 author = {Charles Ehresmann},
 booktitle = {Proceedings of the Fifth Canadian Mathematical Congress: University of Montreal, 1961},
 pages = {109--172},
 publisher = {University of Toronto Press},
 title = {Structures Feuilletées},
 year = {1963},
 address = {Toronto, Ontario, Canada},
}

@article{feff2016,
author = {Fefferman, Charles and Mitter, Sanjoy and Narayanan, Hariharan},
year = {2013},
month = {10},
pages = {},
title = {Testing the Manifold Hypothesis},
volume = {29},
journal = {Journal of the American Mathematical Society},
doi = {10.1090/jams/852}
}

@article{gf2021, title={Model-Centric Data Manifold: The Data Through the Eyes of the Model}, volume={15}, ISSN={1936-4954}, DOI={10.1137/21M1437056}, abstractNote={We show that deep ReLU neural network classifiers can see a low-dimensional Riemannian manifold structure on data. Such structure comes via the local data matrix, a variation of the Fisher information matrix, where the role of the model parameters is taken by the data variables. We obtain a foliation of the data domain, and we show that the dataset on which the model is trained lies on a leaf, the data leaf, whose dimension is bounded by the number of classification labels. We validate our results with some experiments with the MNIST dataset: paths on the data leaf connect valid images, while other leaves cover noisy images.}, number={3}, journal={SIAM Journal on Imaging Sciences}, author={Grementieri, Luca and Fioresi, Rita}, year={2022}, month=sep, pages={1140–1156}, language={en} }

@inproceedings{hanin2019,
  title = {Deep {{ReLU Networks Have Surprisingly Few Activation Patterns}}},
  booktitle = {Advances in {{Neural Information Processing Systems}}},
  author = {Hanin, Boris and Rolnick, David},
  year = 2019,
  volume = {32},
  publisher = {Curran Associates, Inc.},
  address = {Vancouver, Canada}
}

@incollection{hermann63,
title = {On the Accessibility Problem in Control Theory},
editor = {Joseph P. LaSalle and Solomon Lefschetz},
booktitle = {International Symposium on Nonlinear Differential Equations and Nonlinear Mechanics},
publisher = {Academic Press},
address = {Cambridge, Massachusetts, USA},
pages = {325-332},
year = {1963},
isbn = {978-0-12-395651-4},
doi = {https://doi.org/10.1016/B978-0-12-395651-4.50035-0},
author = {Robert Hermann}
}

@article{hinton2006,
author = {G. E. Hinton  and R. R. Salakhutdinov},
title = {Reducing the Dimensionality of Data with Neural Networks},
journal = {Science},
volume = {313},
number = {5786},
pages = {504-507},
year = {2006},
doi = {10.1126/science.1127647}}

@inproceedings{hinton2003,
 author = {Hinton, Geoffrey E and Roweis, Sam},
 booktitle = {Advances in Neural Information Processing Systems},
 editor = {S. Becker and S. Thrun and K. Obermayer},
 pages = {},
 publisher = {MIT Press},
 title = {Stochastic Neighbor Embedding},
 address = {Cambridge, Massachusetts, USA},
 volume = {15},
 year = {2002}
}

@article{lavau2017,
title = {A short guide through integration theorems of generalized distributions},
journal = {Differential Geometry and its Applications},
volume = {61},
pages = {42-58},
year = {2018},
issn = {0926-2245},
doi = {https://doi.org/10.1016/j.difgeo.2018.07.005},

author = {Sylvain Lavau}
}

@ARTICLE{lecun1998,
  author={Lecun, Y. and Bottou, L. and Bengio, Y. and Haffner, P.},
  journal={Proceedings of the IEEE}, 
  title={Gradient-based learning applied to document recognition}, 
  year={1998},
  volume={86},
  number={11},
  pages={2278-2324},
  doi={10.1109/5.726791}}

@article{lenail,
	doi = {10.21105/joss.00747},
	journal = {Journal of Open Source Software},
	title={NN-SVG: Publication-Ready Neural Network Architecture Schematics.},
	author={LeNail, Alexander},
}

@inproceedings{maillet,
author = {Sun, Ke and Marchand-Maillet, Stephane},
year = {2015},
month = {06},
pages = {},
title = {An Information Geometry of Statistical Manifold Learning},
volume = {2},
journal = {31st International Conference on Machine Learning, ICML 2014},
booktitle = {International Conference on Machine Learning},
}

@article{martens2020,
  author  = {James Martens},
  title   = {New Insights and Perspectives on the Natural Gradient Method},
  journal = {Journal of Machine Learning Research},
  year    = {2020},
  volume  = {21},
  number  = {146},
  pages   = {1--76},
}

@article{maurer2016,
  author  = {Andreas Maurer and Massimiliano Pontil and Bernardino Romera-Paredes},
  title   = {The Benefit of Multitask Representation Learning},
  journal = {Journal of Machine Learning Research},
  year    = {2016},
  volume  = {17},
  number  = {81},
  pages   = {1--32},
}

@book{murphy2022,
 author = "Kevin P. Murphy",
 title = "Probabilistic Machine Learning: An introduction",
 publisher = "MIT Press",
 year = 2022,
 address = {Cambridge, Massachusetts, USA}
}

@article{nagano66,
author = {Tadashi Nagano},
title = {{Linear differential systems with singularities and an application to transitive Lie algebras}},
volume = {18},
journal = {Journal of the Mathematical Society of Japan},
number = {4},
publisher = {Mathematical Society of Japan},
pages = {398 -- 404},
year = {1966},
doi = {10.2969/jmsj/01840398},
}

@article{olah,
  title = {Feature Visualization},
  volume = {2},
  ISSN = {2476-0757},
  
  DOI = {10.23915/distill.00007},
  number = {11},
  journal = {Distill},
  publisher = {Distill Working Group},
  author = {Olah,  Chris and Mordvintsev,  Alexander and Schubert,  Ludwig},
  year = {2017},
  month = nov 
}

@Inbook{rao1945,
author="Rao, C. Radhakrishna",
title="Information and the Accuracy Attainable in the Estimation of Statistical Parameters",
bookTitle="Breakthroughs in Statistics: Foundations and Basic Theory",
year="1992",
publisher="Springer New York",
address="New York, NY",
pages="235--247",
isbn="978-1-4612-0919-5",
doi="10.1007/978-1-4612-0919-5_16",
}

@article{reeb,
     author = {Reeb, Georges},
     title = {Sur les structures feuillet\'ees de co-dimension un et sur un th\'eor\`eme de {M.A.} {Denjoy}},
     journal = {Annales de l'Institut Fourier},
     pages = {185--200},
     publisher = {Institut Fourier},
     address = {Grenoble},
     volume = {11},
     year = {1961},
     doi = {10.5802/aif.113},
     zbl = {0136.20901},
     mrnumber = {131283},
     language = {fr},
}

@article{suss1973,
 ISSN = {00029947},
 author = {Héctor J. Sussmann},
 journal = {Transactions of the American Mathematical Society},
 pages = {171--188},
 publisher = {American Mathematical Society},
 title = {Orbits of Families of Vector Fields and Integrability of Distributions},
 volume = {180},
 year = {1973}
}

@article{szalai2023,
  title = {Data-Driven Reduced Order Models Using Invariant Foliations,  Manifolds and Autoencoders},
  volume = {33},
  ISSN = {1432-1467},
  
  DOI = {10.1007/s00332-023-09932-y},
  number = {5},
  journal = {Journal of Nonlinear Science},
  publisher = {Springer Science and Business Media LLC},
  author = {Szalai,  Robert},
  year = {2023},
  month = jun 
}

@inproceedings{ten1998,
 author = {Tenenbaum, Joshua},
 booktitle = {Advances in Neural Information Processing Systems},
 editor = {M. Jordan and M. Kearns and S. Solla},
 pages = {},
 publisher = {MIT Press},
 title = {Mapping a Manifold of Perceptual Observations},
 address = {Cambridge, Massachusetts, USA},
 volume = {10},
 year = {1997}
}

@article{ten2000,
author = {Joshua B. Tenenbaum  and Vin de Silva  and John C. Langford },
title = {A Global Geometric Framework for Nonlinear Dimensionality Reduction},
journal = {Science},
volume = {290},
number = {5500},
pages = {2319-2323},
year = {2000},
doi = {10.1126/science.290.5500.2319}}

@article{tron2022,
  title={Adversarial attacks on neural networks through canonical Riemannian foliations},
  author={Eliot Tron and Nicolas Couellan and St{\'e}phane Puechmorel},
  journal={Machine Learning},
  year={2024},
  volume={113},
}

@article{tronCartanMovingFrames2024,
  title = {Cartan Moving Frames and the Data Manifolds},
  author = {Tron, Eliot and Fioresi, Rita and Cou{\"e}llan, Nicolas and Puechmorel, St{\'e}phane},
  year = 2024,
  month = nov,
  journal = {Information Geometry},
  volume = {7},
  pages = {883--912},
  issn = {2511-249X},
  doi = {10.1007/s41884-024-00159-8},
  langid = {english}
}

@book{tu2008,
  title = {An Introduction to Manifolds},
  ISBN = {9781441974006},
  ISSN = {2191-6675},
  
  DOI = {10.1007/978-1-4419-7400-6},
  journal = {Universitext},
  publisher = {Springer New York},
  author = {Tu,  Loring W.},
  year = {2011},
  address = {New York, NY, USA}
}

@book{tu2017,
  title = {Differential Geometry},
  ISBN = {9783319550848},
  ISSN = {2197-5612},
  
  DOI = {10.1007/978-3-319-55084-8},
  journal = {Graduate Texts in Mathematics},
  address = {Berlin, Germany},
  publisher = {Springer International Publishing},
  author = {Tu,  Loring W.},
  year = {2017}
}

@article{weiss2016,
  title = {A survey of transfer learning},
  volume = {3},
  ISSN = {2196-1115},
  
  DOI = {10.1186/s40537-016-0043-6},
  number = {1},
  journal = {Journal of Big Data},
  publisher = {Springer Science and Business Media LLC},
  author = {Weiss,  Karl and Khoshgoftaar,  Taghi M. and Wang,  DingDing},
  year = {2016},
  month = may 
}

@misc{fashionmnist,
      title={Fashion-MNIST: a Novel Image Dataset for Benchmarking Machine Learning Algorithms}, 
      author={Han Xiao and Kashif Rasul and Roland Vollgraf},
      year={2017},
      eprint={1708.07747},
      archivePrefix={arXiv},
      primaryClass={cs.LG},
       
}

@misc{kmnist,
  author       = {Tarin Clanuwat and Mikel Bober-Irizar and Asanobu Kitamoto and Alex Lamb and Kazuaki Yamamoto and David Ha},
  title        = {Deep Learning for Classical Japanese Literature},
  date         = {2018-12-03},
  year         = {2018},
  eprintclass  = {cs.CV},
  eprinttype   = {arXiv},
  eprint       = {cs.CV/1812.01718},
}

@misc{emnist,
      title={EMNIST: an extension of MNIST to handwritten letters}, 
      author={Gregory Cohen and Saeed Afshar and Jonathan Tapson and André van Schaik},
      year={2017},
      eprint={1702.05373},
      archivePrefix={arXiv},
      primaryClass={cs.CV},
       
}

@article{cifar,
author = {Krizhevsky, Alex},
year = {2012},
month = {05},
pages = {},
title = {Learning Multiple Layers of Features from Tiny Images},
journal = {University of Toronto}
}

@inproceedings{Sun2017RelativeFI,
  title={Relative Fisher Information and Natural Gradient for Learning Large Modular Models},
  author={Ke Sun and Frank Nielsen},
  booktitle={International Conference on Machine Learning},
  year={2017},
}

@misc{Sun2024,
      title={A Geometric Modeling of Occam's Razor in Deep Learning}, 
      author={Ke Sun and Frank Nielsen},
      year={2025},
      eprint={1905.11027},
      archivePrefix={arXiv},
      primaryClass={cs.LG},
       
}

@inproceedings{Hanin2019ComplexityOL,
  title={Complexity of Linear Regions in Deep Networks},
  author={Boris Hanin and David Rolnick},
  booktitle={International Conference on Machine Learning},
  year={2019},
}

@inproceedings{Montfar2014OnTN,
  title={On the Number of Linear Regions of Deep Neural Networks},
  author={Guido Mont{\'u}far and Razvan Pascanu and Kyunghyun Cho and Yoshua Bengio},
  booktitle={Neural Information Processing Systems},
  year={2014},
}

@inproceedings{Hua2023DynamicFO,
  title={Dynamic Flows on Curved Space Generated by Labeled Data},
  author={Xinru Hua and Truyen V. Nguyen and Tam Le and Jos{\'e} H. Blanchet and Viet Anh Nguyen},
  booktitle={International Joint Conference on Artificial Intelligence},
  year={2023},
}

\end{document}